%% file: paper.tex
\documentclass[10pt]{article}
\usepackage[utf8]{inputenc}
\usepackage[T1]{fontenc}
\usepackage{microtype}
\usepackage[margin=1in,left=1.25in,right=1.25in,top=1in,bottom=1in]{geometry}
\usepackage{algorithm}
\usepackage{algorithmic}
\usepackage[linktoc=all,bookmarksdepth=3,pdfstartview=FitH]{hyperref}
\hypersetup{colorlinks=true,linkcolor=blue,filecolor=magenta,urlcolor=cyan}
\usepackage{amsmath}
\usepackage{amsfonts}
\usepackage{amssymb}
\usepackage{graphicx}
\usepackage{placeins}
\usepackage[dvipsnames]{xcolor}
\usepackage{booktabs}
\usepackage{multirow}
\usepackage{amsthm}
\usepackage{enumitem}

\newtheorem{proposition}{Proposition}[section]
\newtheorem{theorem}[proposition]{Theorem}

\newcommand{\resacbaseline}{RE-SAC-0}

\newcommand{\gain}{G}
\newcommand{\clip}{\operatorname{clip}}

\title{BAPR: Belief-Aware Policy Reparameterization for Persistent Actuator Regimes\footnote{Code available at \url{https://github.com/erzhu419/BAPR}}}

\author{
  Yifan Zhang\thanks{Central South University, \href{mailto:erzhu419@gmail.com}{erzhu419@gmail.com}, \href{mailto:204201048@csu.edu.cn}{204201048@csu.edu.cn}} \and
  Liang Zheng\thanks{Central South University, \href{mailto:zhengliang@csu.edu.cn}{zhengliang@csu.edu.cn}} \and
  Qifan Zhang\thanks{Faculty of IT, Monash Malaysia, \href{mailto:zhangqifan0419@gmail.com}{zhangqifan0419@gmail.com}}
}
\date{}

\begin{document}
\maketitle

\begin{abstract}
Piecewise-stationary control is often approached by learning a context-conditioned policy or by switching among mode-specific policies. These approaches entangle two problems: identifying the current dynamics and learning a compatible controller. We study a structured but practically relevant alternative in which persistent actuator regimes differ by known action-coordinate transforms that preserve the command bounds. We introduce \textbf{BAPR} (Belief-Aware Policy Reparameterization), which keeps one canonical policy, infers the current regime from completed transitions, and transforms each canonical action into the current actuator coordinates. A five-head inverse-dynamics ensemble estimates the nominal pre-noise actuator signal from $(s_t,s_{t+1})$; a tempered sticky mode filter accumulates causal evidence; and the maximum-belief mode determines an action-coordinate compensation before the next action.

In a prospectively registered HalfCheetah study with ten new policy seeds and new evaluation streams, every learned control arm receives 8.4 million policy-environment interactions. BAPR additionally uses a separately pretrained, mode-supervised estimator. BAPR obtains a switching return of 3297.1, compared with 2287.1 for SAC, 2046.0 for recurrent ESCP, and 2360.2 for a zero-regularization RE-SAC ablation (\resacbaseline{}). Paired mean improvements are 1009.9 (95\% CI 391.7--1628.2), 1251.1 (553.5--1948.7), and 936.9 (337.1--1536.7), respectively, with policy-seed wins of 7/10, 9/10, and 9/10. The causal estimator reaches 98.55\% mode accuracy and a median first-correct switch delay of three steps. The stricter preregistered composite decision remains negative because SAC consistency and the oracle-headroom seed count fall below their frozen requirements. A separate preregistered Reacher confirmation uses five new policy seeds, 1.6 million interactions per controller, and the same frozen estimator for compensation and belief-conditioned SAC. Compensation improves switching return by 7.64 (95\% CI 5.32--9.97), winning on all five seeds. Follow-up tests retain the Reacher advantage across three estimator initializations and extend it to non-diagonal channel permutations on five new policy seeds (paired improvement 4.67; 95\% CI 2.92--6.43). These results support causal compensation for the tested known, identifiable actuator-coordinate regimes.
\end{abstract}


\section{Introduction}

Robotic and transportation systems rarely operate under one fixed transition model. Actuators can be recalibrated, channels can fail, traffic can enter a new demand regime, and the active dynamics can remain stable for a period before changing abruptly. This motivates the piecewise-stationary Markov decision process, in which an unobserved regime indexes the transition kernel and changes within an episode~\cite{Padakandla2020Survey,Lecarpentier2019NonStationary}.

Most adaptive reinforcement-learning methods address this setting by conditioning a policy on recent experience or by selecting among separately trained policies~\cite{Luo2022ESCP,Rakelly2019PEARL,Zintgraf2020VariBAD}. This is appropriate when different regimes require genuinely different decision rules. It is unnecessarily difficult, however, when regimes are related by a known coordinate transformation. Independently optimized policies can develop incompatible gaits even when their physical objectives are equivalent. A learned router must then identify the regime and switch between controllers whose internal conventions are not aligned. The resulting discontinuity can be more damaging than the regime change itself.

This paper separates \emph{mode inference} from \emph{control}. We consider persistent actuator regimes whose physical input is a known transform of the commanded action plus small actuator noise. One canonical policy is trained in a fixed reference coordinate system; the primary HalfCheetah study initializes it from a multimode SAC source. At deployment, an inverse-dynamics model estimates which transform is active from completed transitions, a tempered sticky mode filter accumulates evidence, and the canonical command is reparameterized into the inferred coordinates. The policy itself is never switched or adapted online.

This design emerged from a sequence of controlled failures. A policy bank with a strong causal router performed well on one HalfCheetah protocol, but a post hoc mechanism audit showed that a single reference policy with exact action compensation was both simpler and stronger. We therefore froze the compensation mechanism, trained fresh reference policies, and conducted a new ten-seed prospective study. The study was designed to distinguish three possible explanations: insufficient mode inference, insufficient oracle headroom, and unstable canonical-controller training.

Our contributions are:
\begin{itemize}[leftmargin=*]
    \item \textbf{A structured adaptation mechanism.} We formulate persistent actuator non-stationarity as an action-coordinate problem and derive an exact compensation map when the inverse command is feasible. The tested polarity and signed-permutation transforms preserve the command box. The method uses one canonical controller rather than a bank of independently optimized policies.
    \item \textbf{Causal online identification.} A transition-only inverse-dynamics ensemble and tempered sticky mode belief infer the active transform. The action at time $t$ uses only transitions completed before $t$.
    \item \textbf{A prospective, equal-policy-budget evaluation.} Ten new HalfCheetah policy seeds and disjoint event streams compare BAPR with SAC~\cite{Haarnoja2018SAC}, recurrent ESCP~\cite{Luo2022ESCP}, and \resacbaseline{}, the zero-regularization ablation of RE-SAC~\cite{Zhang2026RESAC}, each trained with 8.4 million policy-environment interactions.
    \item \textbf{An independent same-information comparison.} Five new Reacher policy seeds compare explicit compensation with belief-conditioned SAC using the same frozen inverse model, causal belief definition, historical estimator cost, and controller-interaction budget. The preregistered primary paired interval is positive. Separate follow-ups test three estimator initializations and a non-diagonal transform family with fresh policy and estimator seeds.
    \item \textbf{An explicit evidence boundary.} BAPR has statistically resolved mean gains over all three comparators, but the composite gate misses both the SAC consistency counts and the required oracle-headroom seed count. Ant, Hopper, Walker2d, and bus evidence is reported as transfer or scope evidence rather than pooled into the HalfCheetah claim.
\end{itemize}

\section{Related work}

\subsection{Non-stationary and context-conditioned reinforcement learning}

Non-stationary RL includes continual learning, hidden-parameter MDPs, change-point methods, and meta-RL~\cite{Khetarpal2020ContinualRL,Doshi-Velez2016HiPMDP,Adams2007BOCD,Nagabandi2019MAML}. Context-conditioned methods infer a latent variable from recent experience and feed it to the actor or critic. ESCP learns an environment-sensitive context for sudden changes~\cite{Luo2022ESCP}; PEARL and VariBAD infer task beliefs for meta-RL~\cite{Rakelly2019PEARL,Zintgraf2020VariBAD}. These methods jointly learn representation and control. BAPR instead exploits a known transformation family: the latent variable selects a coordinate map around a single policy.

Bayesian online change detection maintains a posterior over change histories~\cite{Adams2007BOCD}. BAPR does not use the original run-length-based BOCD and lower-confidence-bound machinery from earlier development versions. Its deployment rule is a finite-state sticky filter whose prior and evidence are tempered. The normalized belief is therefore a decision statistic, not a claim of calibrated hidden Markov model probabilities. This distinction matters because the current claim concerns causal system identification and action reparameterization, not adaptive critic conservatism.

\subsection{Robust and ensemble reinforcement learning}

Robust MDP methods optimize against transition uncertainty sets~\cite{Iyengar2005RobustDP,Wiesemann2013RobustMDP}, while distributionally robust actor-critic methods extend this idea to function approximation~\cite{Zhou2023NaturalActorCritic,Panaganti2024RobustF}. SAC provides the maximum-entropy off-policy baseline~\cite{Haarnoja2018SAC}. RE-SAC separates aleatoric regularization from epistemic ensemble disagreement and was originally developed for stochastic bus control~\cite{Zhang2026RESAC}. Robust training can produce a useful fallback policy, but a single compromise policy may leave performance on the table when persistent modes require conflicting actuator commands. BAPR asks whether structured compensation can recover that mode-specific headroom without learning multiple deployment policies.

\subsection{System identification and policy reparameterization}

Hidden-parameter and contextual RL infer dynamics variables that are not directly observed~\cite{Doshi-Velez2016HiPMDP,Benjamins2021CARL}. UP-OSI predicts dynamics parameters from recent state--action history and passes them to a universal policy~\cite{Yu2017UPOSI}; RMA similarly combines a base policy with a learned adaptation module~\cite{Kumar2021RMA}. BAPR does not claim that separating identification from control is new. Its narrower distinction is that the candidate actuator transforms are known and act only as coordinate changes, so the inferred mode selects an analytic map around one canonical policy rather than conditioning a learned policy on an estimated parameter.

Known state--action symmetries can also be built into equivariant policy and value networks~\cite{VanderPol2020MDPHomomorphic}. Those methods assume the relevant symmetry is available when the network is evaluated. BAPR instead treats the active group element as hidden and time varying: it must be inferred causally before the corresponding action-coordinate map is applied. This distinction provides a true-mode reference that separates controller headroom from inference loss.

\section{Problem formulation}
\label{sec:problem}

We consider a piecewise-stationary MDP with continuous state $s_t\in\mathcal{S}$ and commanded action $a_t\in[-1,1]^d$. A latent regime $z_t\in\{0,\ldots,M-1\}$ remains fixed for a dwell interval and may change within an episode. We distinguish the nominal pre-noise actuator signal
\begin{equation}
    v_t = \clip\!\left(\gain_{z_t}a_t,-1,1\right)
    \label{eq:nominal_signal}
\end{equation}
from the action actually passed to the physics step,
\begin{equation}
    u_t = \clip\!\left(\gain_{z_t} a_t + \varepsilon_t,-1,1\right),
    \label{eq:executed_action}
\end{equation}
where $\gain_z$ is a known invertible actuator-coordinate transform and $\varepsilon_t$ is zero-mean actuator noise. The transition and reward are generated from $(s_t,u_t)$. The agent observes $(s_t,a_t,r_t,s_{t+1})$ but not $v_t$, $u_t$, or $z_t$. The primary benchmarks use diagonal polarity transforms; a separate Reacher extension uses signed permutations.

The primary benchmark uses $M=4$ polarity transforms. For HalfCheetah's six action dimensions, the affected channel sets are the lower half, upper half, even indices, and odd indices. Affected channels have gain $-1$ and the remaining channels gain $+1$. Thus each $\gain_z$ is diagonal, orthogonal, and self-inverse: $\gain_z^{-1}=\gain_z$. All modes share the same robot physics and additive actuator-noise standard deviation 0.02. The mode persists for 250 environment steps.

The objective is not to learn an arbitrary universal policy. We seek a controller that (i) preserves the performance of a policy trained in a canonical regime, (ii) infers the active transform causally, and (iii) avoids discontinuities caused by switching among independently optimized policies.

\section{BAPR}
\label{sec:method}

\subsection{Canonical policy and exact action compensation}

Let $r$ denote a fixed reference regime and let $\pi_r(s)$ be a deterministic deployment action from a policy trained in that coordinate system. If the current regime $z$ were known, BAPR would command
\begin{equation}
    a_t^{\mathrm{comp}}
    = \clip\!\left(\gain_z^{-1}\gain_r\pi_r(s_t),-1,1\right).
    \label{eq:compensation_general}
\end{equation}
To reproduce the reference transformed command before noise and actuator clipping, an invertible transform also requires a feasible command. With $\mathcal A=[-1,1]^d$ and canonical action $q\in\mathcal A$,
\begin{equation}
 \gain_z\clip(\gain_z^{-1}\gain_r q,-1,1)=\gain_r q
 \quad\Longleftrightarrow\quad
 \gain_z^{-1}\gain_r q\in\mathcal A.
 \label{eq:compensation_feasibility}
\end{equation}
Multiplying by $\gain_z^{-1}$ reduces this condition to $\clip(w,-1,1)=w$, which holds exactly when $w\in\mathcal A$. Invertibility alone therefore does not ensure feasibility.
For polarity transforms, this becomes
\begin{equation}
    a_t^{\mathrm{comp}}
    = \gain_z\gain_r\pi_r(s_t),
    \label{eq:compensation_polarity}
\end{equation}
because sign changes preserve the action bounds.

\begin{proposition}[Pathwise equivalence under arbitrary polarity switching]
\label{prop:equivalence}
Let $(z_t)_{t=0}^{H-1}$ be any regime sequence. Assume that all regimes share the same transition and reward maps and differ only through diagonal sign matrices $\gain_z$. Couple a rollout that remains in reference regime $r$ and a rollout that follows $(z_t)$ by using the same initial simulator state and exogenous-noise sequence. If the latter commands Eq.~\eqref{eq:compensation_polarity} with the true $z_t$ at every step, then the two rollouts have identical executed actions, states, rewards, and finite-horizon return pathwise. Consequently, their return distributions and expectations are equal under the same noise law.
\end{proposition}
\begin{proof}
At equal states, the deterministic canonical policy produces the same reference action. Because $\gain_{z_t}^2=I$, the compensated rollout has
$\gain_{z_t}a_t^{\mathrm{comp}}=\gain_{z_t}\gain_{z_t}\gain_r\pi_r(s_t)=\gain_r\pi_r(s_t)$.
Sign matrices preserve the action box, and adding the coupled noise before the common componentwise clipping therefore gives the same executed action as in regime $r$. The common transition and reward maps give the same next state and reward. Induction over $t$ proves the pathwise claim; equality in distribution follows by taking the same coupling under the common noise law. Appendix~\ref{app:proofs} gives the full construction and assumptions.
\end{proof}

For a signed-permutation transform, $\gain_z^{-1}=\gain_z^\top$ rather than necessarily $\gain_z$. Both $\gain_z$ and its transpose preserve the command box. Thus Eq.~\eqref{eq:compensation_general} is feasible and gives $\gain_z a_t^{\mathrm{comp}}=\gain_r\pi_r(s_t)$; the same coupled-noise induction proves pathwise equivalence. This matrix argument covers the channel-permutation extension, whereas the componentwise error formula below is specific to polarity transforms.

The true-mode controller defined by Eq.~\eqref{eq:compensation_polarity} is a privileged-information reference for the learned estimator and exactly reproduces the reference-regime trajectory under the coupling above. It is not a universal pointwise upper bound on every incorrectly routed action sequence. The uncompensated canonical policy is a necessary lower control: it tests whether the benchmark actually requires the transform.

\begin{proposition}[Action-weighted consequence of a mode error]
\label{prop:mode_error}
For polarity transforms, let $q_t=\pi_r(s_t)$, let $z_t$ be the true mode, and let $\hat z_t$ be the mode used for compensation. The resulting nominal-signal error relative to reference execution is
\begin{equation}
 e_t=\left(\gain_{z_t}\gain_{\hat z_t}-I\right)\gain_r q_t,
 \qquad
 \|e_t\|_2^2=4\!\sum_{j:g_{z_t,j}\ne g_{\hat z_t,j}}q_{t,j}^2,
 \label{eq:weighted_mode_error}
\end{equation}
where $g_{z,j}$ is diagonal entry $j$ of $\gain_z$. Under common actuator noise, componentwise clipping cannot increase this same-state action deviation.
\end{proposition}
Proposition~\ref{prop:mode_error} shows why unweighted mode accuracy is incomplete: an error on an inactive actuator has little immediate effect, whereas an error on a high-magnitude channel can be costly. It does not by itself bound long-horizon return because later states may amplify an action error. Appendix~\ref{app:proofs} gives the derivation.

\subsection{Transition-only inverse system identification}

Neither the nominal signal nor the noisy executed action is exposed to the agent. Rather than fitting one realization of $\varepsilon_t$, we train an ensemble $\{f_j\}_{j=1}^{K}$ to predict the nominal signal from the observed transition:
\begin{equation}
    \hat v_t^{(j)} = f_j\!\left(
        \tanh(s_t/5),\tanh(s_{t+1}/5),
        \tanh((s_{t+1}-s_t)/0.1)
    \right).
\end{equation}
During offline estimator training, the known training-mode label constructs $v_t$ using Eq.~\eqref{eq:nominal_signal}. Independently bootstrapped heads minimize
\begin{equation}
 \mathcal{L}_{\mathrm{inv}}
 =\frac{1}{K}\sum_{j=1}^{K}
 \mathbb{E}_{(s_t,a_t,z_t,s_{t+1})\sim\mathcal{B}_j}
 \left[\frac{1}{d}\|\hat v_t^{(j)}-v_t\|_2^2\right].
 \label{eq:inverse_loss}
\end{equation}
The estimator does not receive the commanded action or mode label as an input. The command is used only after prediction to score candidate regimes. For candidate mode $z$, the implied nominal signal is
\begin{equation}
    \tilde v_t(z)=\clip(\gain_z a_t,-1,1).
\end{equation}
With empirical diagonal residual variance $\sigma^2$, the ensemble emission score is
\begin{equation}
    \ell_t(z)=\log\frac{1}{K}\sum_{j=1}^{K}
    \exp\left[-\frac{1}{2d}\sum_{q=1}^{d}
    \left(
      \frac{(\tilde v_{t,q}(z)-\hat v_{t,q}^{(j)})^2}{\sigma_q^2}
      +\log(2\pi\sigma_q^2)
    \right)\right].
    \label{eq:emission}
\end{equation}
The empirical variance absorbs actuator noise, inverse-model error, and other transition ambiguity; it is not identified with the physical noise variance $0.02^2$. The frozen five-head estimator uses three-layer multilayer perceptrons with width 256. In the primary HalfCheetah study, its final 1,500-update stage is a warm-started fine-tuning stage, not the complete estimator-training cost. Appendix~\ref{app:estimator} reports the preceding 3,000 updates, trajectory sources, supervised transition counts, and validation data. These separately trained components are outside the equal 8.4M policy-interaction budget. Reacher uses a separately fitted and frozen estimator, whose cost is shared by both information-matched routes (Appendix~\ref{app:reacher_protocol}).

\paragraph{Identifiability.}
Known transforms must also be distinguishable from the observed transitions. Under the same causal behavior policy and initial-state law, two modes that produce identical transformed commands on every reachable history generate identical observed trajectory laws. An equal-prior binary classifier can then do no better than chance. For polarity modes, a distinguishing channel must carry a nonzero command, and its effect must be visible in the transition. Appendix~\ref{app:causal_identification} gives this argument. Prediction errors and switch delays are evaluated empirically rather than inferred from invertibility or actuator noise.

\subsection{Tempered sticky mode filter and causal execution}

Let $b_t(z)$ be the normalized mode belief available before selecting action $a_t$. A sticky transition prior with hazard $h$ is
\begin{equation}
    \bar b_t(z)=(1-h)b_t(z)+\frac{h}{M-1}\left(1-b_t(z)\right).
\end{equation}
After transition $t$ is completed, BAPR updates
\begin{equation}
    b_{t+1}(z)\propto \bar b_t(z)^{\delta}
    \exp\!\left(c\,\ell_t(z)\right),
    \label{eq:posterior}
\end{equation}
with $h=0.002$, evidence scale $c=1$, and prior exponent $\delta=0.98$. Because $\delta\ne1$, this is a tempered, forgetting update rather than an unmodified hidden Markov model posterior; calibration is evaluated empirically. The next action uses $\hat z_{t+1}=\arg\max_z b_{t+1}(z)$. This ordering is causal: $a_t$ uses $b_t$, which contains evidence only from transitions before $t$. No true mode ID, current next state, or future transition enters action selection.

With fixed initialization, frozen measurable maps and deterministic tie-breaking, induction makes the command a function of the pre-action history alone. Appendix~\ref{app:causal_identification} states the information ordering. The filter's forgetting and sticky prior are deployment choices; its calibration and recovery are measured in the mechanism diagnostics.

\begin{algorithm}[t]
\caption{BAPR deployment with a frozen canonical policy}
\label{alg:bapr}
\begin{algorithmic}[1]
\STATE Initialize $b_0(z)=1/M$ and observe $s_0$
\FOR{$t=0,1,\ldots$}
    \STATE $\hat z_t\leftarrow\arg\max_z b_t(z)$
    \STATE $a_t^{\mathrm{ref}}\leftarrow\pi_r(s_t)$
    \STATE $a_t\leftarrow\gain_{\hat z_t}^{-1}\gain_r a_t^{\mathrm{ref}}$
    \STATE Execute $a_t$ and observe $(r_t,s_{t+1})$
    \STATE Predict $\{\hat v_t^{(j)}\}_{j=1}^{K}$ from $(s_t,s_{t+1})$
    \STATE Compute $\ell_t(z)$ for every mode using Eq.~\eqref{eq:emission}
    \STATE Update $b_{t+1}$ using Eq.~\eqref{eq:posterior}
\ENDFOR
\end{algorithmic}
\end{algorithm}

\begin{theorem}[Finite-horizon consequence of causal mode errors]
\label{thm:return_gap}
In the polarity benchmark, fix a horizon $H$, an exogenous mode schedule $z_{0:H-1}$, and a common initial-state distribution. Let $x_t$ be the complete simulator state, with observed state $s_t$ supplied to $\pi_r$. The true-mode oracle $O$ and causal BAPR controller $C$ share the physics, reward map, reset convention, and actuator-noise law of Eq.~\eqref{eq:executed_action}. Assume that their cumulative rewards and the oracle action-value terms below are integrable. Let $Q_t^O(x,v)$ be the expected reward at $t$ plus the subsequent return under $O$ when nominal pre-noise signal $v$ is used at state $x$; it averages over the current actuator noise and all future randomness. Put $v_t^O=\gain_r\pi_r(s_t)$ and $v_t^C=\gain_{z_t}\gain_{\hat z_t}\gain_r\pi_r(s_t)$. Then, without any smoothness assumption,
\begin{equation}
 J_C-J_O=\sum_{t=0}^{H-1}\mathbb E_C\!\left[
 Q_t^O(x_t,v_t^C)-Q_t^O(x_t,v_t^O)\right].
 \label{eq:causal_return_identity}
\end{equation}
If, additionally, $Q_t^O(x,\cdot)$ is $L_t(x)$-Lipschitz between these two signals at states visited by $C$, and the displayed sensitivity budget is integrable, then
\begin{equation}
 |J_C-J_O|\leq 2\sum_{t=0}^{H-1}\mathbb E_C\!\left[
 L_t(x_t)\left(\sum_{j:g_{z_t,j}\ne g_{\hat z_t,j}}
 \pi_r(s_t)_j^2\right)^{1/2}\right].
 \label{eq:causal_return_bound}
\end{equation}
\end{theorem}
The identity applies to the 1,000-step HalfCheetah and original 200-step Reacher polarity protocols under its stated integrability assumptions. The bound weights errors by both actuator excitation and action-value sensitivity along the causal trajectory. This explains why high aggregate mode accuracy can coexist with a consequential switch transient. Appendix~\ref{app:return_proof} proves both statements.

Figure~\ref{fig:v32_method} summarizes the separation between persistent actuator coordinates, causal mode inference, and action reparameterization.

\begin{figure}[t]
\centering
\includegraphics[width=\textwidth]{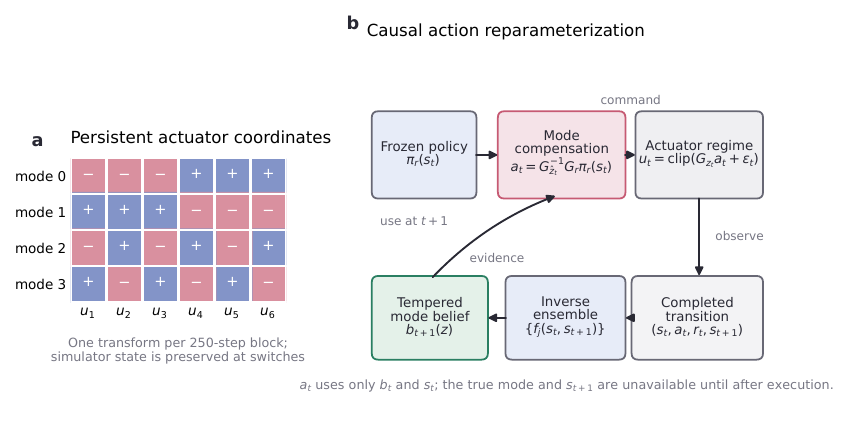}
\caption{BAPR mechanism for persistent actuator-polarity regimes. \textbf{a}, The four known transforms alter the signs of fixed actuator channels while preserving robot physics, noise scale, and simulator state across 250-step blocks. \textbf{b}, The frozen canonical policy emits an action in reference coordinates. The belief available before step $t$ selects the inverse transform used for $a_t$. Only after execution does the frozen inverse ensemble process $(s_t,s_{t+1})$; its predictions and the completed command update the tempered belief $b_{t+1}$, which can affect the next action but not the current one.}
\label{fig:v32_method}
\end{figure}

\subsection{Controller training}

For each prospective policy seed, the canonical path first trains a SAC source policy for 5.6M interactions over the actuator-regime family and then fine-tunes the complete actor--critic state for 2.8M interactions in fixed reference mode 0. The actor, critic, target critic, and temperature are inherited; the replay buffer and optimizers are reset at the stage boundary. The final policy is used without checkpoint selection. The resulting 8.4M interactions and 525,000 gradient updates match the budgets of SAC, recurrent ESCP, and \resacbaseline{}. BAPR does not update the policy, critic, or estimator during evaluation.

\section{Experiments}
\label{sec:experiments}

\subsection{Questions and protocol}

We ask four questions: (Q1) Does causal compensation improve switching return over equal-budget SAC, ESCP, and \resacbaseline{}? (Q2) How much of the same canonical controller's true-mode return does causal inference recover? (Q3) How consistent is the improvement across independently trained policies and event streams? (Q4) Where does the mechanism cease to apply?

The primary experiment is a prospective HalfCheetah-v2 confirmation registered before training. Here ``HalfCheetah-v2'' is the repository's legacy name for Brax HalfCheetah with the spring backend and the custom actuator-regime wrapper. It is not the OpenAI Gym MuJoCo-v2 implementation. All compared methods use this same evaluation environment. Execution-profile provenance is described in Appendix~\ref{app:source_diagnostic}. A five-seed fresh-policy pilot was used only to select a sample size of ten and is not pooled with the confirmation. The confirmation uses ten new policy seeds, three new stationary event streams, and three new switching event streams. Each switching episode lasts 1,000 steps and visits all four modes in 250-step blocks; schedules and starting modes vary by event stream and episode. Stationary evaluation fixes one mode for the full episode. Policies are evaluated deterministically, while actuator noise remains active.

The learned arms are equal in policy-environment interactions: SAC, recurrent ESCP, \resacbaseline{}, and the BAPR canonical path each receive 8.4M interactions and 525,000 updates per policy seed. The \resacbaseline{} arm retains five critics but sets both \texttt{weight\_reg} and \texttt{beta\_ood} to zero; its settings are given in Appendix~\ref{app:protocol}. BAPR's additional estimator data and supervision are accounted for separately in Appendix~\ref{app:estimator} and the limitations below. We additionally report the 5.6M robust source, the uncompensated canonical policy, and true-mode compensation. The primary statistic is the paired switching-return difference across policy seeds. Each seed's return is averaged over the three event streams before inference, so the statistical sample size is ten policy seeds, not 30 seed--event cells. Two-sided 95\% confidence intervals use a paired $t$ interval with nine degrees of freedom; they are marginal intervals for individual comparisons, without a multiple-comparison adjustment.

The preregistered success gate requires, against each equal-budget comparator: (i) a positive paired mean; (ii) a positive lower 95\% confidence bound; (iii) at least 8/10 policy-seed wins; and (iv) at least 24/30 seed--event wins. Mechanism checks additionally require at least 10\% true-mode headroom over equal-budget SAC, recovery of at least 70\% of the no-compensation-to-oracle gain, and at least 95\% stationary oracle retention for at least eight seeds. We report the frozen composite decision unchanged, while interpreting the mean-and-interval criteria as evidence about average effect and the win counts as evidence about cross-seed consistency.

\subsection{Primary confirmation results}

\begin{table}[t]
\centering
\caption{Prospective confirmation returns averaged over ten new HalfCheetah policy seeds. Every learned final control arm has an 8.4M policy-interaction budget. The 5.6M source is diagnostic.}
\label{tab:v32_returns}
\begin{tabular}{lrr}
\toprule
Method & Switching & Stationary \\
\midrule
Robust SAC source, 5.6M & 1650.8 & 1775.0 \\
SAC, 8.4M & 2287.1 & 2382.6 \\
Canonical, no compensation & 340.3 & 560.3 \\
Canonical + true mode & 3422.0 & 3432.5 \\
\textbf{BAPR, causal compensation} & \textbf{3297.1} & \textbf{3403.5} \\
Recurrent ESCP, 8.4M & 2046.0 & 2117.0 \\
\resacbaseline{}, 8.4M & 2360.2 & 2435.9 \\
\bottomrule
\end{tabular}
\end{table}

Table~\ref{tab:v32_returns} shows a large gap between the uncompensated canonical policy and both compensated variants. For this frozen controller, the gap isolates the benefit of correcting the actuator coordinates: its mode-0 actions transfer poorly when interpreted under another sign transform. Causal BAPR recovers most of the true-mode result, reaching 3297.1 switching return and 3403.5 stationary return.

\begin{table}[t]
\centering
\caption{Paired prospective switching comparisons. ``Seeds'' counts positive policy-seed differences; ``events'' counts positive seed--event cells.}
\label{tab:v32_paired}
\begin{tabular}{lrrrrc}
\toprule
Comparison & Difference & 95\% CI & Seeds & Events & Gate \\
\midrule
BAPR $-$ no compensation & $+2956.8$ & $[2500.7,3412.8]$ & 10/10 & 30/30 & pass \\
BAPR $-$ source SAC, 5.6M & $+1646.2$ & $[1174.9,2117.6]$ & 10/10 & 30/30 & pass \\
BAPR $-$ SAC, 8.4M & $+1009.9$ & $[391.7,1628.2]$ & 7/10 & 23/30 & fail \\
BAPR $-$ ESCP, 8.4M & $+1251.1$ & $[553.5,1948.7]$ & 9/10 & 25/30 & pass \\
BAPR $-$ \resacbaseline{}, 8.4M & $+936.9$ & $[337.1,1536.7]$ & 9/10 & 27/30 & pass \\
\bottomrule
\end{tabular}
\end{table}

The paired mean advantage is statistically resolved for every configured equal-budget comparator (Table~\ref{tab:v32_paired} and Fig.~\ref{fig:v32_results}). Relative to comparator means, BAPR improves switching return by 44.2\% over SAC, 61.1\% over ESCP, and 39.7\% over \resacbaseline{}, with policy-seed wins of 7/10, 9/10, and 9/10. The evidence therefore supports a positive average effect against these three configured arms, but not uniform superiority. The preregistered composite decision remains \textbf{FAIL} for two reasons: the SAC comparison reaches 7/10 policy-seed wins and 23/30 seed--event wins rather than 8/10 and 24/30, and at least 10\% true-mode headroom over SAC occurs for 7/10 rather than 8/10 seeds. We retain these frozen outcomes rather than redefining the gate after observing the data.

\begin{figure}[t]
\centering
\includegraphics[width=\textwidth]{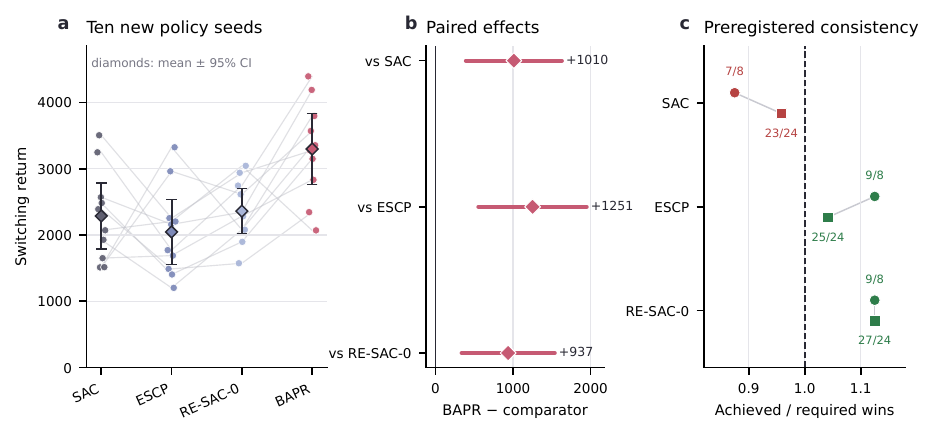}
\caption{Prospective confirmation results on ten new HalfCheetah policy seeds. \textbf{a}, Switching return for the four equal-budget learned arms. Thin lines connect the same policy-seed index across independently trained methods; diamonds show the mean and two-sided 95\% $t$ interval over $n=10$ policy seeds. \textbf{b}, Paired BAPR-minus-comparator mean differences and 95\% intervals over the same ten seeds. \textbf{c}, Consistency counts relative to the preregistered requirements. Circles denote policy-seed wins out of the required eight and squares denote seed--event wins out of the required 24; the dashed line is the passing threshold. The SAC comparison reaches 7/8 and 23/24 of the required counts. This fails the frozen consistency component despite a positive paired interval.}
\label{fig:v32_results}
\end{figure}

\paragraph{Runtime-matched follow-up.}
A separate five-new-policy-seed comparison retained the frozen mechanism and estimator while recording matching training and audit software versions.
BAPR obtained 3611.7 switching return versus 2005.5 for SAC, 2059.1 for recurrent ESCP and 2218.2 for \resacbaseline{}.
All three secondary paired intervals were positive, with five of five policy-seed wins each.
This follow-up reused the original evaluation streams and is not pooled with the prospective ten-seed result.
Appendix~\ref{app:runtime_comparison} reports its protocol, intervals and failed regularized transfer arm.

\subsection{Mechanism diagnostics}

The causal estimator attains 98.55\% mode accuracy, a Brier score of 0.0246, and a median first-correct switch-detection delay of three steps. Compensation recovery passes for 10/10 policy seeds, and stationary retention passes for 10/10. Proposition~\ref{prop:equivalence} is also verified numerically in the preceding mechanism audit: with paired simulator states and noise, true-mode compensation has zero maximum nominal-signal and return error relative to native reference-mode execution.

For each seed, the final advantage decomposes exactly as
\begin{equation}
J_{\mathrm{causal}}-J_{\mathrm{SAC}}
=\underbrace{(J_{\mathrm{oracle}}-J_{\mathrm{SAC}})}_{\text{controller headroom}}
-\underbrace{(J_{\mathrm{oracle}}-J_{\mathrm{causal}})}_{\text{inference loss}}.
\label{eq:return_decomposition}
\end{equation}
At the aggregate means, true-mode headroom is 1134.9 and inference loss is 124.9, leaving the observed 1009.9 advantage. Theorem~\ref{thm:return_gap} expresses this oracle-minus-causal gap as action-value differences at states visited by the causal controller; accuracy and detection delay alone do not determine its sign or magnitude. The three SAC non-wins occur at policy seeds 88021, 88039, and 88179. Their true-mode headroom over SAC is only $+2.3\%$, $-2.3\%$, and $+4.1\%$. Limited controller headroom is therefore the dominant issue, but it is not the only issue: residual inference loss reverses the two small positive margins at seeds 88021 and 88179. Figure~\ref{fig:v32_diagnostics} displays this decomposition directly.

\begin{figure}[t]
\centering
\includegraphics[width=\textwidth]{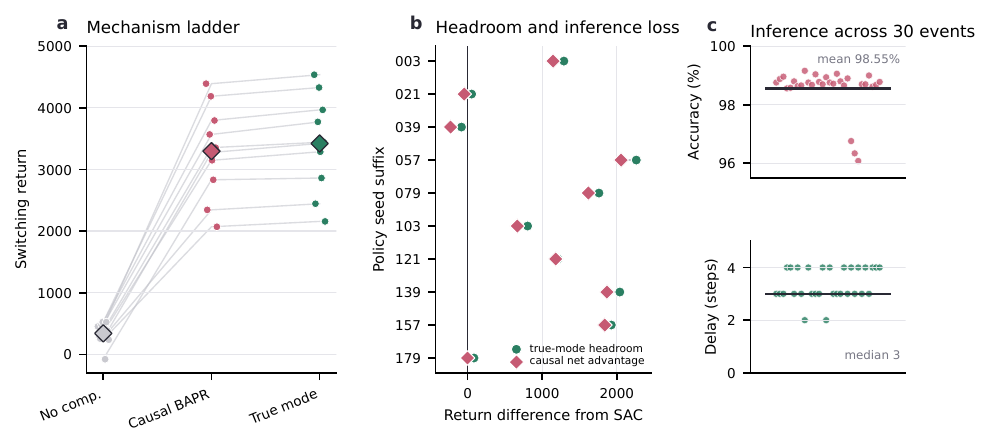}
\caption{Mechanism diagnostics separate controller headroom from inference loss. \textbf{a}, Per-seed switching returns for the uncompensated canonical policy, causal BAPR, and true-mode compensation; thin lines connect the same canonical policy and diamonds show means. \textbf{b}, For each policy seed, the true-mode point is oracle headroom over equal-budget SAC and the causal point is the final BAPR advantage. Their connecting segment is inference loss; the vertical zero line separates wins from non-wins. \textbf{c}, Points show mode accuracy and the median first-correct switch delay within each of 30 policy-seed--event cells. Horizontal bars show mean event accuracy (98.55\%) and the median of event-level median delays (3 steps); pooling all 450 switch delays also gives a median of 3 steps.}
\label{fig:v32_diagnostics}
\end{figure}

\FloatBarrier
\subsection{Independent same-information comparison on Reacher}
\label{sec:reacher_confirmation}

Explicit compensation also outperformed a learned policy with access to the same causal estimator in an independent Reacher confirmation. This study retained the native Brax generalized-backend task and used four two-actuator sign transforms, a 200-step horizon, 50-step dwell intervals, and actuator-noise standard deviation 0.02. Five new policy seeds and new evaluation streams were registered after a separate three-seed development comparison. Earlier cohorts were not pooled.

Each seed trained four controllers from scratch: a fixed-reference SAC policy, a memoryless multimode robust SAC policy, a filtered-belief-conditioned SAC policy, and an instantaneous-belief-conditioned SAC policy. Every controller received 1.6M interactions and 100,000 updates. BAPR applied compensation around the reference policy; the conditioned policies instead received $(s_t,b_t)$ and learned the command directly. Both matched routes used the same frozen inverse model and pre-action belief definition during deployment, and conditioned SAC used that definition during training. Beliefs were updated only after completed transitions. Their numerical histories could differ because the controllers chose different actions. The primary outcome was filtered compensation minus filtered conditioned SAC; the instantaneous pair was a predefined secondary comparison. Appendix~\ref{app:reacher_protocol} gives the full protocol and shared estimator costs.

\begin{table}[t]
\centering
\caption{Independent Reacher confirmation across five new policy seeds. Four separately trained controllers each receive 1.6M interactions and 100,000 updates. Both compensation variants reuse the same reference policy. Stationary and switching evaluations are separate.}
\label{tab:reacher_returns}
\begin{tabular}{lrr}
\toprule
Method & Switching & Stationary \\
\midrule
True-mode compensation / native reference & -4.457 & -4.379 \\
\textbf{BAPR, filtered compensation} & \textbf{-6.759} & \textbf{-5.683} \\
Filtered belief-conditioned SAC & -14.400 & -13.390 \\
Instantaneous compensation & -6.467 & -6.217 \\
Instantaneous belief-conditioned SAC & -10.395 & -11.785 \\
Memoryless robust SAC & -27.079 & -23.482 \\
\bottomrule
\end{tabular}
\end{table}

The primary mean difference was $+7.641$ (paired 95\% CI $[5.317,9.965]$), with five of five policy-seed wins (Table~\ref{tab:reacher_returns}, Fig.~\ref{fig:reacher_confirmation}). Its positive lower confidence bound met the frozen confirmation criterion. The predefined instantaneous comparison was also positive: $+3.928$ ($[2.816,5.040]$), with five of five wins. Filtered compensation exceeded memoryless robust SAC by $20.320$ ($[18.357,22.283]$) and recovered 89.83\% of the oracle-minus-robust gap at the group means. All 975 switching route episodes retained 200 steps and three switches without termination, and all true-mode returns matched their paired native-reference returns exactly.

\begin{figure}[t]
\centering
\includegraphics[width=\textwidth]{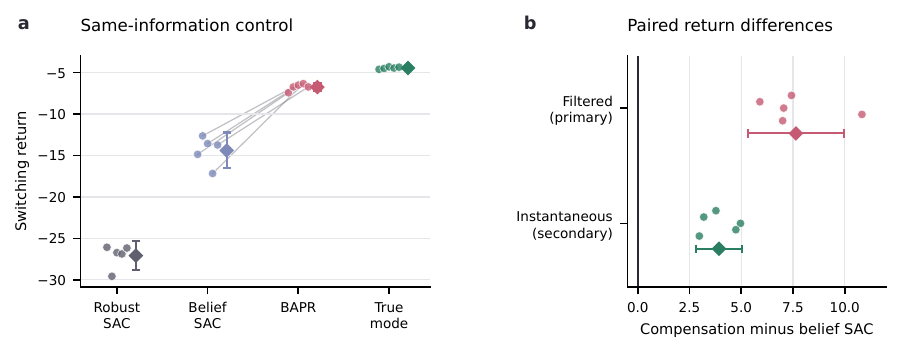}
\caption{Independent information-matched Reacher confirmation. \textbf{a}, Switching returns for five new policy seeds; circles show seed means, and diamonds and bars show the group mean and two-sided 95\% $t$ interval over $n=5$ policy seeds. Thin lines pair filtered belief-conditioned SAC with BAPR at the same seed. \textbf{b}, Seed-level compensation-minus-conditioned-SAC differences and their paired mean intervals. The filtered comparison is primary and the instantaneous comparison is predefined secondary. Each seed averages 15 paired episodes; episodes are repeated measurements, not independent statistical samples.}
\label{fig:reacher_confirmation}
\end{figure}

The component comparisons did not establish a separate return benefit from temporal filtering or ensembling. Filtered minus instantaneous compensation was $-0.292$ (95\% CI $[-1.142,0.558]$), despite higher mode accuracy (81.00\% versus 51.19\%). Five-head compensation minus fixed inverse head 0 was $-0.058$ ($[-0.151,0.035]$). An earlier independent five-seed Reacher comparison found a negative filtered-minus-instantaneous difference of $-0.541$ ($[-0.751,-0.331]$); that cohort is reported separately rather than pooled. Thus the confirmed comparison concerns the policy interface, not a claim that each inference component improves return.

\FloatBarrier
\subsection{Estimator initializations and non-diagonal coordinates}
\label{sec:coordinate_extension}

The positive Reacher comparison persisted across three independently initialized inverse ensembles. Each estimator was frozen after 4,500 updates on the same prescribed calibration data; each conditioned-SAC comparator used its own matched estimator during training and evaluation. On new event streams, the five existing reference and robust policies were crossed with these estimators, and the ten missing matched conditioned policies were trained from scratch. Every estimator gave a positive paired interval and five of five policy-seed wins (Table~\ref{tab:coordinate_extension}). Estimator seeds are crossed factors, not fifteen independent policy replicates.

\begin{table}[ht]
\centering
\small
\setlength{\tabcolsep}{4pt}
\caption{Reacher follow-ups with five policy seeds per comparison. Differences are BAPR minus matched belief-conditioned SAC, with marginal paired 95\% intervals. The polarity rows reuse the original policy cohort on new events; the permutation row uses fresh policies and its preselected estimator.}
\label{tab:coordinate_extension}
\begin{tabular}{lrrlc}
\toprule
Study/model & BAPR & Conditioned SAC & Difference [95\% CI] & Wins \\
\midrule
Polarity/model 1 & -7.015 & -14.626 & $7.611\ [5.736,9.486]$ & 5/5 \\
Polarity/model 2 & -7.094 & -13.839 & $6.745\ [5.467,8.023]$ & 5/5 \\
Polarity/model 3 & -6.918 & -15.085 & $8.167\ [6.290,10.044]$ & 5/5 \\
Permutation/primary & -6.436 & -11.109 & $4.674\ [2.920,6.427]$ & 5/5 \\
\bottomrule
\end{tabular}
\end{table}

An independent extension tested four predefined signed permutations, including two non-diagonal channel mappings, with five fresh policy seeds and three fresh estimator initializations. It retained Reacher physics, Gaussian noise of standard deviation 0.02, 200-step episodes and 50-step dwell intervals. The preselected estimator gave switching returns of $-6.436$ for BAPR, $-11.109$ for its matched conditioned SAC, $-25.828$ for robust SAC and $-4.986$ for true-mode compensation. The other two estimators gave $-6.413$ and $-6.449$; they were not selected to replace the primary. True-mode compensation exceeded every fixed context on all five seeds and exactly reproduced paired native-reference returns. The primary causal model detected 223 of 225 switches, with median first-correct delay 12 steps among detected switches and two unresolved switches. Appendix~\ref{app:coordinate_extension} gives seeds, supervision and coverage.

\FloatBarrier
\subsection{Developmental and transfer evidence}
\label{sec:scope_evidence}

The prospective ten-seed HalfCheetah confirmation supports the primary three-comparator claim; the independent Reacher confirmation addresses the distinct same-information question. Earlier experiments explain why this architecture was selected but are not pooled with either study. A five-seed mechanism audit reused frozen policies and found 4076.2 switching return for causal compensation versus 1424.8 for robust SAC and 3246.0 for a causally routed policy bank. A subsequent fresh-policy pilot trained five new policy seeds; its mean was favorable, but confidence intervals against SAC and \resacbaseline{} crossed zero. The pilot was therefore used only for prospective sample-size planning.

Transfer results define the current scope. A three-seed Ant true-mode audit showed return headroom for compensation (4510.7 versus 2588.6 for robust SAC), but failed its absolute termination-safety gate, so no causal Ant claim is made. Hopper and Walker2d headroom screens terminated before the first scheduled switch in all tested episodes, making those protocols invalid tests of adaptation. The stochastic bus environment remains an important positive domain for RE-SAC-derived robustness~\cite{Zhang2026RESAC,Zhang2026SingleAgentBus}, but it does not expose the known invertible action symmetry required by Eq.~\eqref{eq:compensation_general}; its results are therefore not presented as evidence for the action-compensation mechanism.

\section{Discussion}

\subsection{What the result establishes}

The central empirical finding is that the tested persistent actuator-regime problem becomes substantially easier after choosing the correct representation. In the tested polarity and signed-permutation families, the environment modes are equivalent up to an action-coordinate transform. Using one canonical policy removes the need to switch between independently optimized controllers and preserves a common control rule. The causal recovery and switch-delay measurements show that online identification can support this mechanism in these tasks.

The oracle ladder is essential. Without the no-compensation and true-mode controls, a strong mean return could be attributed either to a better policy seed or to successful mode inference. Here, no compensation fails sharply, true-mode compensation supplies a privileged empirical ceiling for the same canonical policy, and the learned estimator approaches that reference. This ceiling is empirical rather than a mathematical upper bound over arbitrary action sequences. Equation~\eqref{eq:return_decomposition} further shows that final performance depends on both available controller headroom and the inference loss spent recovering it. Perfect inference cannot create headroom that training did not produce, while small residual inference errors can still erase a small positive margin.

The Reacher comparison distinguishes the compensation interface from merely supplying mode information to a learner. The analytic interface outperformed both filtered and instantaneous belief-conditioned SAC while sharing the inverse model, its offline supervision, and the controller-interaction budget. The initialization audit retained the matched advantage across three models, and the non-diagonal study extended it beyond channelwise sign flips. The original component tests also exposed a difference between identification metrics and control utility: temporal filtering improved mode accuracy without an established return benefit. The empirical role of the belief is therefore to support useful action correction, not to maximize classification accuracy in isolation.

\subsection{Limitations}

The current evidence has five clear limits. First, the confirmations cover HalfCheetah and Reacher with known box-preserving polarity or signed-permutation matrices. These are stronger assumptions than unknown mass, friction, terrain, morphology, delays or irreversible faults. The positive same-information confirmation is specific to Reacher; a separate five-seed HalfCheetah comparison was inconclusive (Appendix~\ref{app:reacher_protocol}). Second, the estimators use separately trained trajectory providers and supervised mode labels. HalfCheetah equalizes policy interactions, not total end-to-end data or prior information. Reacher equalizes the historical estimator cost between compensation and conditioned SAC, but not against memoryless robust SAC, and the conditioned actor has 70,916 parameters versus 69,892 for the reference actor. Its comparison changes both the policy interface and fixed-reference versus multimode training curriculum, rather than isolating one architectural component.

Third, the original HalfCheetah confirmation remains conditional on one frozen estimator. Reacher now includes three estimator initializations in each transform family, but their shared calibration data do not test different provider distributions. A retrospective HalfCheetah audit of three models trained with a different calibration recipe retained switching returns of 2702.8--2773.3 versus oracle 3298.1; it is not multi-initialization replication of the original two-stage estimator (Appendix~\ref{app:hc_calibration_sensitivity}). Fourth, the HalfCheetah preregistered composite gate fails despite positive paired confidence intervals, so BAPR has a strong average advantage with controller-seed heterogeneity, not uniform superiority. Its \resacbaseline{} comparison is a zero-regularization ablation. A subsequent original Torch-trainer comparison also favors BAPR on the common Brax benchmark, but changes the source paper's collector and protocol, retains fewer actor updates and excludes BAPR's estimator cost from the matched budget (Appendix~\ref{app:source_trainer_comparison}). It does not establish superiority under the original RE-SAC paper protocol. Reacher does not establish an additional return benefit from filtering or ensembling.

Fifth, the two confirmed tasks lack health-termination failures in these evaluations. The Ant audit shows that return headroom does not by itself establish safe transfer. A post-confirmation five-seed controller-training diagnostic did not consistently improve on the frozen HalfCheetah recipe (Appendix~\ref{app:additional_results}); it is not pooled with the prospective result.

The original ten-seed HalfCheetah training signatures span two interpreter profiles without archived dependency versions (Appendix~\ref{app:source_diagnostic}), so a runtime confound cannot be excluded for that cohort.
The separate five-seed follow-up records matching versions and corroborates the SAC, ESCP and \resacbaseline{} mean advantages, but reuses the old evaluation streams and does not retrospectively establish historical runtime identity (Appendix~\ref{app:runtime_comparison}).
Its regularized JAX transfer recipe fails and differs from the original RE-SAC implementation, so that contrast does not establish superiority over the source method.
The independent Reacher confirmation records matching dependency versions for its controllers and audits.

The theoretical return bound is conditional on integrability and local oracle action-value sensitivity. These conditions are not verified globally for the Brax simulator, and the theorem does not calculate a numerical return guarantee from mode accuracy. Its role is to identify the action-weighted inference loss, not to certify empirical superiority over other algorithms.

Broader gain families remain an open applicability test. Any such extension must check the inverse-command feasibility condition in Eq.~\eqref{eq:compensation_feasibility} and mode distinguishability in Appendix~\ref{app:causal_identification}; mere matrix invertibility is insufficient. Reproducing the original HalfCheetah estimator recipe across independent initializations is a separate remaining robustness question. Extending BAPR to arbitrary dynamics changes would require a different mechanism, not another threshold adjustment to the current posterior.

\section{Conclusion}

We presented BAPR, a belief-aware policy reparameterization method for persistent actuator regimes. BAPR separates causal mode inference from control by combining a frozen canonical policy, a transition-only inverse-dynamics ensemble, a tempered sticky mode belief, and exact action-coordinate compensation for feasible transforms. In a prospective ten-seed HalfCheetah study, BAPR has statistically resolved mean gains over equal-budget SAC, recurrent ESCP, and the zero-regularization \resacbaseline{} ablation, with majority seed wins against each configured arm. The preregistered composite decision remains negative because both SAC consistency and the number of seeds with sufficient oracle headroom miss their frozen thresholds. An independent five-seed Reacher confirmation establishes a positive paired advantage over belief-conditioned SAC using the same frozen estimator and controller-interaction budget; separate tests retain that advantage across three estimator initializations and extend it to non-diagonal signed permutations. Together, these results support causal compensation for the tested known, identifiable, reversible actuator regimes. Controller headroom and action-weighted inference loss, rather than mode accuracy alone, determine its control benefit.

\section*{Acknowledgments}
The authors thank the developers of JAX, Brax, and the open-source reinforcement-learning ecosystem used in this study.

\section*{Code and data availability}
The code repository contains the prospective registration, analysis and plotting scripts, together with a compact release package. The package includes the ten per-seed audit records, the aggregate analysis, the frozen inverse-model parameters, and the ten frozen canonical-policy parameter sets. These records reproduce the reported tables and figures without raw replay buffers or full training checkpoints. The compact package does not by itself reproduce training or rerun every baseline rollout; those require the stated software environment, training scripts, and separately trained baseline policies.

A separate Reacher audit package contains its preregistration, five per-policy-seed JSON audits, aggregate analysis, and frozen estimator/provider metadata. It supports recomputation of the reported returns, paired intervals, and figure without training checkpoints or calibration arrays. Training and rollout regeneration use the stated software versions and separately retained model parameters.

The runtime-matched HalfCheetah follow-up is released separately with its preregistration, five audits, aggregate analysis and thirty training-completion metadata records.
A further compact package supplies per-episode returns, paired statistics and budget records for the estimator-initialization, non-diagonal-coordinate, retrospective calibration and source-trainer studies. It supports statistical recomputation without calibration arrays or checkpoints; it does not contain the models needed for fresh rollouts.
The accompanying standalone Lean package supplies the fourteen environment-independent mechanism sources, a toolchain specification, a Lake build configuration and pinned dependencies.

\input{appendix_submission}

\input{references_v32}
\end{document}

%% file: appendix_submission.tex
\clearpage
\appendix

\section{Actuator-regime benchmark}
\label{app:benchmark}

The prospective confirmation benchmark keeps HalfCheetah morphology, gravity, reward, observation space, and episode horizon fixed. Only the actuator-coordinate transform changes. Before each physics step, the environment applies
\begin{equation}
    u_t=\clip(\gain_{z_t}a_t+\varepsilon_t,-1,1),
    \qquad \varepsilon_t\sim\mathcal{N}(0,0.02^2 I).
\end{equation}
For six action dimensions, the four fixed transforms are:
\begin{center}
\begin{tabular}{clc}
\toprule
Mode & Affected channels & $\operatorname{diag}(\gain_z)$ \\
\midrule
0 & lower half & $(-1,-1,-1,+1,+1,+1)$ \\
1 & upper half & $(+1,+1,+1,-1,-1,-1)$ \\
2 & even indices & $(-1,+1,-1,+1,-1,+1)$ \\
3 & odd indices & $(+1,-1,+1,-1,+1,-1)$ \\
\bottomrule
\end{tabular}
\end{center}

The environment holds the active mode fixed for 250 steps. A switching episode contains four blocks and lasts 1,000 steps. The simulator state is preserved at a switch; the robot is not reset. Each mode therefore defines a persistent stochastic transition distribution rather than a per-step resampling of robot morphology.

\subsection{Execution-profile provenance}
\label{app:source_diagnostic}

Frozen training signatures record the resolved interpreter paths.
A retrospective inspection on 28 September 2026 found Brax 0.10.5 with JAX 0.4.38 in the three GPU-server virtual environments.
The recorded shared HPC interpreter contains Brax 0.12.1 with JAX 0.5.3; the local analysis environment contains Brax 0.14.2 with JAX 0.9.2.
Each inspected resolver retains the flattened scalar-inertia gather.
These inspections describe the recorded paths at the inspection date, not historical dependency snapshots.

The primary ten-seed confirmation also spans these two training profiles.
Nine canonical-policy signatures use the GPU-server interpreter and one uses the shared HPC interpreter; the corresponding counts are four/six for SAC, six/four for recurrent ESCP, and five/five for \resacbaseline{}.
The compact release retains each comparator's resolved interpreter, host, configuration and final budget.
Because dependency versions were not archived during these runs, equal interaction budgets do not establish identical historical simulator implementations, and the primary HalfCheetah comparison retains this possible runtime confound.

A read-only diagnostic reused canonical actor seed 88003 and switching-event identifier 197201 in the current local runtime.
Each original branch executed 1,000 actions and 16,000 physics substeps from a normal reset.
Neither branch produced a nonpositive effective denominator, but penetrating-contact normal denominators fell below inverse mass 1,046 times with true-mode compensation and 341 times without compensation.
A private resolver copy changes only the inertia gather to the matrix axis and expands its world-link mask over both matrix dimensions.
On this single actor--event case, true-mode-compensated return changes from 4213.7 to 4512.0; uncompensated return changes from 1051.9 to 729.4.
The original installed modules and prospective-result files remain unchanged.
This isolated runtime diagnostic is not pooled with the frozen confirmation.

\section{Proofs for the action-compensation mechanism}
\label{app:proofs}

The proofs use common transition and reward maps, a frozen canonical policy, and a known family of actuator transforms.

\subsection{Feasibility and pathwise equivalence}

For an invertible $\gain_z$, set $w=\gain_z^{-1}\gain_r q$.
Multiplying $\gain_z\clip(w)=\gain_r q$ by $\gain_z^{-1}$ gives $\clip(w)=w$.
Componentwise clipping fixes exactly the points in $[-1,1]^d$, proving Eq.~\eqref{eq:compensation_feasibility}.
Polarity transforms preserve that box, so every canonical action gives an admissible compensated command.

\begin{proof}[Proof of Proposition~\ref{prop:equivalence}]
Let $x_t$ denote the complete simulator state and $s_t$ its observation.
Couple reference and compensated rollouts with equal initial states and the same exogenous noise, including any transition or reset randomness.
Assume their complete states agree at step $t$.
The frozen policy then gives the same canonical action $q_t$ in both rollouts, and
\begin{equation}
 \gain_{z_t}a_t^{\mathrm{comp}}
 =\gain_{z_t}\gain_{z_t}\gain_r q_t
 =\gain_r q_t.
\end{equation}
The commands are admissible because sign matrices preserve the action box.
Adding the shared actuator noise and applying the common actuator clipping therefore gives equal executed actions.
The common transition and reward maps give equal next states and rewards.
Induction proves equality for every step, for any polarity-mode sequence.
Returns agree pathwise, so their distributions and integrable expectations agree under the common noise law.
\end{proof}

\subsection{Action-weighted mode error}

\begin{proof}[Proof of Proposition~\ref{prop:mode_error}]
At the same simulator state, the nominal signal under estimated mode $\hat z_t$ is
$v_t^C=\gain_{z_t}\gain_{\hat z_t}\gain_r q_t$, whereas the reference signal is $v_t^O=\gain_r q_t$.
Their difference is
\begin{equation}
 e_{t,j}=(g_{z_t,j}g_{\hat z_t,j}-1)g_{r,j}q_{t,j}.
\end{equation}
The multiplier is zero when the mode signs agree and has magnitude two otherwise.
Since $g_{r,j}^2=1$, summing the squared coordinate errors proves Eq.~\eqref{eq:weighted_mode_error}.
Componentwise clipping is non-expansive, so adding common noise gives
\begin{equation}
 \|\clip(v_t^C+\varepsilon_t)-\clip(v_t^O+\varepsilon_t)\|_2
 \le \|e_t\|_2.
\end{equation}
\end{proof}

\subsection{Causal execution and distinguishability}
\label{app:causal_identification}

Before action $t$, the available information is
\begin{equation}
 \mathcal H_t=\sigma\!\left(s_0,(a_\tau,r_\tau,s_{\tau+1})_{\tau=0}^{t-1}\right).
\end{equation}
The initialization $b_0=1/M$ is fixed.
With measurable frozen maps and fixed tie-breaking, $(b_t,s_t)$ determines $\hat z_t$ and $a_t$.
After execution, the completed transition determines the new predictions, scores and $b_{t+1}$.
Induction therefore makes $a_t$ measurable with respect to $\mathcal H_t$; the current $s_{t+1}$ first affects action $t+1$.

To compare two fixed modes, use the same initial-state law and the same mode-blind causal behavior policy.
If their transformed commands agree on every history reachable under either hypothesis, the common noise law gives the same executed-action distribution.
The common physics, observation and reward maps then give the same next observed transition law.
Induction over the dwell interval makes the two observed trajectory laws identical.
For any binary decision rule $\phi$ under equal priors, their common trajectory law $P$ gives
\begin{equation}
 \Pr(\phi\text{ is correct})
 =\tfrac12 P(\phi=z)+\tfrac12 P(\phi=z')\le\tfrac12.
\end{equation}
In the polarity family, a channel on which two modes differ contributes no distinction while its commanded action is zero.
Nonzero commands must also affect observable transitions to permit identification.

\subsection{Finite-horizon return identity and conditional bound}
\label{app:return_proof}

\begin{proof}[Proof of Theorem~\ref{thm:return_gap}]
Fix the exogenous mode schedule and horizon $H$.
Let $x_t$ be the complete physical state and define the oracle values by
\begin{align}
 V_H^O(x)&=0,\\
 Q_t^O(x,v)&=\mathbb E\!\left[r_t+V_{t+1}^O(x_{t+1})
   \mid x_t=x,\text{ nominal signal }v\right],\\
 V_t^O(x)&=Q_t^O(x,v_t^O(x)).
\end{align}
These values include actuator noise, any other transition randomness and the evaluation reset convention.
The primary HalfCheetah confirmation has no health terminations.
Controller memory can be added to the state; the oracle ignores this memory, and its physical-state value remains unchanged.
This permits the causal controller to depend on its completed-transition history.

Conditioning on that history and its chosen signal gives
\begin{equation}
 \mathbb E_C\!\left[Q_t^O(x_t,v_t^C)-V_t^O(x_t)\right]
 =\mathbb E_C\!\left[r_t+V_{t+1}^O(x_{t+1})-V_t^O(x_t)\right].
\end{equation}
The stated integrability conditions permit summation and telescoping.
The terminal value is zero, and the common initial-state law gives $\mathbb E[V_0^O(x_0)]=J_O$.
Since $V_t^O(x_t)=Q_t^O(x_t,v_t^O)$, this proves Eq.~\eqref{eq:causal_return_identity}.
Every expectation is taken along the causal controller's trajectory rather than a separately sampled oracle trajectory.

Under the stated local sensitivity condition,
\begin{equation}
 |Q_t^O(x_t,v_t^C)-Q_t^O(x_t,v_t^O)|
 \le L_t(x_t)\|v_t^C-v_t^O\|_2.
\end{equation}
Proposition~\ref{prop:mode_error} gives
\begin{equation}
 \|v_t^C-v_t^O\|_2
 =2\left(\sum_{j:g_{z_t,j}\ne g_{\hat z_t,j}}q_{t,j}^2\right)^{1/2}.
\end{equation}
The triangle inequality and the integrable sensitivity budget yield Eq.~\eqref{eq:causal_return_bound}.
An exogenous random schedule can be handled by conditioning on its realization and averaging, provided the same integrability conditions hold jointly.
\end{proof}

The accompanying Lean development formalizes the compensation identities and conditional return arguments.
It is supplied as a standalone package with fourteen environment-independent modules and pinned Lean/Mathlib dependencies.
Its compensation identities concern diagonal polarity transforms; the signed-permutation extension uses the matrix argument following Proposition~\ref{prop:equivalence}.
The recursive Bellman formalization uses sufficient integrability premises over all states and counterfactual signals at finite horizons, stronger than the visited-law conditions used above.
The return bound retains its action-value sensitivity assumption.

\section{Estimator details}
\label{app:estimator}

The frozen inverse model has five independently parameterized heads. Every head is a three-hidden-layer MLP of width 256 with sigmoid linear unit (SiLU) activations and a final $\tanh$. It predicts the nominal pre-noise signal $v_t=\clip(\gain_{z_t}a_t,-1,1)$ rather than the sampled noisy action $u_t$. The transition feature is
\begin{equation}
    x_t=\left[
      \tanh(s_t/5),
      \tanh(s_{t+1}/5),
      \tanh((s_{t+1}-s_t)/0.1)
    \right].
\end{equation}
The released estimator has two training stages. The initialization was trained from scratch for 3,000 updates using robust and true-mode-conditioned controller trajectories from controller seeds 8 and 16. Its training set contained 96,000 transitions: 64,000 stationary and 32,000 switching transitions. The final estimator was warm-started from that model and fine-tuned for 1,500 updates using robust, four fixed-mode specialist, dynamic-oracle, and causal-router trajectories from source seed 4021. This stage contained 82,000 training transitions and 82,000 separate validation transitions. Across both fitting stages, the estimator therefore consumed 178,000 supervised training transitions, excluding the interaction cost used to train the trajectory-producing policies.

Both fitting stages use AdamW with weight decay $10^{-5}$, global gradient-norm clipping at 10, and bootstrap batches of 256 per head. The initialization learning rate is $3\times10^{-4}$ and the final-stage learning rate is $10^{-4}$. True mode is used offline only to construct the nominal-signal target; deployment receives no mode label. The six empirical inverse-residual variances are
\begin{equation}
 (0.1289,0.1030,0.1174,0.1357,0.1063,0.0967).
\end{equation}
Variance is clipped to $[10^{-4},0.25]$ in the model protocol.

The final-stage estimator-training policy seed is 4021 and the validation policy seed is 4049. Training event seeds are 154001 and 154013; validation event seeds are 154101 and 154113. These precede and are disjoint from all pilot and prospective-confirmation policy and evaluation seeds. The prospective protocol does not refit the estimator.

The frozen filter parameters are hazard $0.002$, evidence scale $1.0$, and prior exponent $0.98$. The residual variances include actuator noise and model error and are not estimates of actuator-noise variance alone. The normalized belief used for action $t$ is recorded before incorporating transition $(s_t,a_t,r_t,s_{t+1})$. This convention prevents use of the current next state to choose the current action. The exponent $0.98$ introduces forgetting, so the resulting belief is evaluated as a decision statistic rather than asserted to be a calibrated hidden Markov model posterior.

\section{Prospective training and evaluation protocol}
\label{app:protocol}

\subsection{Policy budgets}

The canonical path has two stages. A source SAC controller trains for 1,400 iterations with 4,000 environment steps and 250 gradient updates per iteration, totaling 5.6M interactions and 350,000 updates. Its full actor, critic, target critic, and temperature state is then fine-tuned in fixed mode 0 for 700 iterations, adding 2.8M interactions and 175,000 updates. Replay and optimizer states are reset at the fork, while these four learned parameter sets are preserved. The final canonical policy therefore receives 8.4M interactions and 525,000 updates. No validation checkpoint is selected; the final policy is evaluated.

SAC, recurrent ESCP, and \resacbaseline{} each receive the same 8.4M interactions and 525,000 updates. This is an equal policy-training budget, not an equal end-to-end data or supervision budget, because BAPR also uses the separately pretrained frozen inverse model and its trajectory-producing policies described in Appendix~\ref{app:estimator}.

\subsection{Comparator configurations}

The frozen \resacbaseline{} arm uses five critics, a target blend of 75\% independent targets and 25\% ensemble-minimum targets, an LCB coefficient annealed from $-2$ to zero after a 20\% warmup, an actor-anchor coefficient of 0.001, and EMA evaluation with update coefficient 0.005. It sets both \texttt{weight\_reg} and \texttt{beta\_ood} to zero and disables behavior cloning. It is the configured zero-regularization RE-SAC ablation, not the regularized algorithm. The released per-seed training metadata records these settings and final budgets for all three comparison arms. This labeling clarifies the original frozen configuration without changing its return records or decision rule.

The SAC arm uses two critics and hidden width 256. The recurrent ESCP arm uses a 16-step history, two critics, hidden width 256, policy/context learning rates $3\times10^{-4}$, critic learning rate $10^{-3}$, and temperature learning rate $10^{-2}$. Its context updates require 100,000 interactions and at least two observed tasks; actor and critic targets use the twin minimum. The temperature is capped at 1 and its target-entropy ratio is 1.5. These are the recorded comparator configurations for the structured actuator benchmark.

\subsection{Seeds and event streams}

The ten prospective-confirmation policy seeds are
\begin{equation}
\begin{split}
\{&88003,88021,88039,88057,88079,\\
  &88103,88121,88139,88157,88179\}.
\end{split}
\end{equation}
Stationary event seeds are 197101, 197117, and 197133. Switching event seeds and their base schedules are:
\begin{center}
\begin{tabular}{cc}
\toprule
Event seed & Base mode schedule \\
\midrule
197201 & $(2,0,3,1)$ \\
197217 & $(1,3,0,2)$ \\
197233 & $(3,2,1,0)$ \\
\bottomrule
\end{tabular}
\end{center}
Each event evaluates five episodes. Episode index cyclically rotates the base schedule so the result is not tied to one starting mode. Stationary evaluation uses five episodes per mode and event seed. All policies use deterministic deployment actions; environment noise remains stochastic.

\subsection{Registered decision rule}

The five-seed fresh-policy pilot failed its registered confirmation rule. It was used only to estimate a sample size of ten for the prospective confirmation; no pilot return is pooled with the confirmation. Before prospective training, the algorithm, reference mode, estimator, policy budgets, seeds, event streams, and decision rule were frozen.

For each comparator, the prospective confirmation requires all of the following:
\begin{enumerate}[leftmargin=*]
    \item positive mean paired switching-return difference;
    \item positive lower bound of a two-sided 95\% paired $t$ interval;
    \item at least 8 of 10 positive policy-seed differences;
    \item at least 24 of 30 positive policy-seed--event differences.
\end{enumerate}
The interval is computed from the ten policy-seed mean differences using the $t_{0.975,9}=2.2622$ critical value. Event wins are consistency diagnostics rather than independent samples in the confidence interval.

The mechanism ladder also checks: (i) true-mode compensation has at least 10\% headroom over equal-budget SAC; (ii) causal compensation recovers at least 70\% of the no-compensation-to-oracle gain; and (iii) causal stationary return retains at least 95\% of oracle stationary return. At least eight seeds must pass each check.

\section{Additional result interpretation}
\label{app:additional_results}

\subsection{Why the registered composite decision fails}

BAPR's mean switching advantages over SAC, ESCP, and \resacbaseline{} all have positive lower 95\% confidence bounds. ESCP and \resacbaseline{} also pass the frozen consistency counts. The composite decision has two blockers. First, the SAC comparison has 7/10 seed wins and 23/30 event wins, below 8/10 and 24/30. Second, true-mode compensation exceeds equal-budget SAC by at least 10\% for 7/10 seeds rather than the required 8/10. The mean-effect evidence is positive, while the registered uniformity and headroom requirements are not fully met.

The three non-winning SAC policy seeds are 88021, 88039, and 88179. Their BAPR-minus-SAC switching differences are $-45.1$, $-228.0$, and $-1.2$. True-mode headroom over SAC is only $+2.3\%$, $-2.3\%$, and $+4.1\%$, respectively. Limited canonical-controller headroom is the dominant cause. However, the causal estimator remains below the true-mode return, and this residual loss reverses the small positive true-mode margins at seeds 88021 and 88179.

For clarity, the reported compensation-recovery metric is
\begin{equation}
 R_{\mathrm{comp}}=
 \frac{J_{\mathrm{causal}}-J_{\mathrm{no\text{-}comp}}}
 {J_{\mathrm{oracle}}-J_{\mathrm{no\text{-}comp}}}.
\end{equation}
It measures recovery of the compensation mechanism, not recovery of oracle advantage over SAC. At the aggregate means, $R_{\mathrm{comp}}=95.9\%$, oracle headroom over SAC is 1134.9, inference loss is 124.9, and the final BAPR advantage over SAC is 1009.9.

The median three-step delay is the first post-switch step at which the estimated mode is correct; it does not require the estimate to remain correct thereafter. Averaged over 30 seed--event cells, wrong-route fractions within the first 1, 2, 4, and 8 post-switch steps are 99.8\%, 99.6\%, 82.7\%, and 46.8\%, respectively. The corresponding overall Brier score is 0.0246 and the mean absolute nominal-signal deviation is 0.0104. These diagnostics expose the unavoidable causal transient more directly than aggregate accuracy alone.

\subsection{Developmental evidence is not pooled}

\begin{table}[ht]
\centering
\caption{Role of the developmental and confirmatory experiments. The first three rows are diagnostic and are not pooled with the primary ten-seed result.}
\label{tab:development_history}
\small
\begin{tabular}{@{}p{0.18\textwidth}p{0.14\textwidth}p{0.27\textwidth}p{0.29\textwidth}@{}}
\toprule
Study & Environment & Purpose & Outcome \\
\midrule
Mechanism audit & HalfCheetah & reused-policy mechanism test & compensation passed \\
Transfer audit & Ant & true-mode transfer and safety & headroom found; safety gate failed \\
Fresh-policy pilot & HalfCheetah & five-seed sample-size pilot & confirmation failed \\
Prospective test & HalfCheetah & ten-seed confirmatory test & mean CIs pass; SAC consistency fails \\
\bottomrule
\end{tabular}
\end{table}

In the reused-policy mechanism audit, causal compensation obtained 4076.2 switching return, compared with 4291.8 for true-mode compensation, 3246.0 for a causally routed policy bank, and 1424.8 for robust SAC. Mode accuracy was 98.84\% and median delay was three steps. Because this audit reused policies selected during development, it establishes mechanism plausibility rather than independent confirmation.

In the fresh-policy pilot, causal compensation obtained 3400.6 switching return, compared with 2594.3 for SAC, 1901.8 for ESCP, and 1916.1 for \resacbaseline{}. The five-seed paired confidence intervals against SAC and \resacbaseline{} crossed zero, so the pilot remained a failed confirmation and only informed the prospective sample size.

\subsection{Post-confirmation controller-training diagnostic}

After the ten-seed confirmation, we compared three canonical-controller recipes and an equal-budget robust SAC control on five new development policy seeds. Each recipe used 8.4M policy interactions and 525,000 updates; the frozen estimator and held-out event streams were shared. The original 5.6M multimode plus 2.8M fixed-reference warm-start obtained 3161.3 causal switching return. Training directly in reference mode 0 for 8.4M interactions obtained 3323.2, but exceeded warm-start on only 3/5 paired seeds (mean difference 162.0; 95\% CI $[-480.3,804.3]$). Canonicalized multimode training obtained 2440.1 and lost to warm-start on all five seeds. In this polarity-only family, canonicalization makes the physical training problem equivalent to fixed mode 0; the result should not be read as a test of canonicalization under additional mode-dependent physics. The equal-budget robust SAC control obtained 2085.5. Under the predeclared 4/5 paired-win rule, the warm-start recipe remained selected. These post-confirmation diagnostics are not additional confirmation seeds and do not change the frozen ten-seed decision.

\subsection{Runtime-matched five-seed follow-up}
\label{app:runtime_comparison}

A separate protocol, frozen on 1 October 2026, retained the original canonical-controller recipe and inverse estimator but required matching JAX 0.4.38, Brax 0.10.5, Flax 0.10.4 and MuJoCo 3.1.6 versions for every training run and audit.
The five new policy seeds were 109003, 109023, 109043, 109063 and 109083.
All final controller paths received 8.4M interactions and 525,000 updates; the canonical path again used a 5.6M source followed by a 2.8M fixed-reference fork.
All thirty training-completion records retain configurations, versions, hosts, interpreter paths and budgets.
Stationary and switching evaluation reused the original three event streams and schedules, rather than introducing new event holdouts.
Every switching route retained the 1,000-step horizon and all four 250-step blocks, without health termination.

\begin{table}[ht]
\centering
\caption{Separate runtime-matched HalfCheetah follow-up across five new policy seeds. The paired differences are BAPR minus the named arm; intervals use policy seeds as the sampling unit. The source remains a 5.6M diagnostic.}
\label{tab:runtime_comparison}
\small
\begin{tabular}{lrrr}
\toprule
Arm & Switching & Difference & 95\% CI \\
\midrule
\textbf{BAPR} & \textbf{3611.7} & --- & --- \\
True-mode compensation & 3827.3 & --- & --- \\
SAC & 2005.5 & 1606.2 & $[571.8,2640.5]$ \\
Recurrent ESCP & 2059.1 & 1552.6 & $[543.3,2561.9]$ \\
\resacbaseline{} & 2218.2 & 1393.5 & $[562.0,2225.0]$ \\
Regularized JAX transfer & -148.0 & 3759.7 & $[2814.1,4705.3]$ \\
No compensation & 513.4 & --- & --- \\
Robust SAC source & 1859.9 & --- & --- \\
\bottomrule
\end{tabular}
\end{table}

The registered primary contrast was BAPR minus a regularized JAX transfer recipe, with both \texttt{weight\_reg} and \texttt{beta\_ood} set to 0.01 and their original positive signs retained.
That recipe failed in this benchmark.
Its five-critic network, optimization choices and reward scale differ from the original RE-SAC implementation; it is not a faithful source replication.
It also differs from the historical \resacbaseline{} recipe in target mixing, behavior cloning, anchor and EMA settings, so the two arms do not isolate the effect of regularization.
The SAC, ESCP and \resacbaseline{} contrasts were predefined secondary marginal paired intervals, and each had five of five policy-seed wins.
These three comparisons corroborate the positive mean effects under matching recorded software versions without relying on the failed transfer arm.
The study matches policy-training budgets, not total estimator supervision or end-to-end data.
It is reported separately and does not change the original ten-seed composite decision.

\subsection{Original Torch trainer on the common benchmark}
\label{app:source_trainer_comparison}

A protocol frozen on 5 October 2026 compared the unmodified original Torch RE-SAC trainer and networks with the existing five canonical policies from Appendix~\ref{app:runtime_comparison}. The source used ten LayerNorm critics, $\beta=-2$, positive \texttt{weight\_reg}=0.01, \texttt{beta\_ood}=0.01, \texttt{beta\_bc}=0.001, gradient clipping at 1 and automatic entropy tuning. The paired no-shift arm changed only \texttt{weight\_reg} to zero; it retained the ensemble and OOD term and is distinct from both SAC and the historical \resacbaseline{} arm.

Both source arms received 8.4M interactions including 10,000 warmup transitions, 525,000 critic calls and 262,500 actor opportunities. The original actor period of two was retained, rather than matching BAPR's 525,000 actor updates. Actual critic optimizer steps ranged from 524,774 to 524,790 because of mixed-precision skips; all 262,500 actor updates ran. Torch was pinned to 2.4.1+cu121, with the same JAX/Brax/Flax/MuJoCo versions as the runtime-matched study. The collector sampled common Brax physics on CPU and trained the source networks on CUDA. Rewards were unscaled, and BAPR's separate historical estimator cost remained additional.

\begin{table}[ht]
\centering
\small
\caption{Original source-trainer comparison on the common HalfCheetah actuator-polarity benchmark. The five policy seeds are reused from the runtime-matched cohort, with new event streams. Differences are BAPR minus each source arm; both paired comparisons win on all five seeds.}
\label{tab:source_trainer_comparison}
\begin{tabular}{lrrr}
\toprule
Route & Switching & Difference & 95\% CI \\
\midrule
\textbf{BAPR} & \textbf{3632.6} & --- & --- \\
True-mode compensation & 3828.0 & --- & --- \\
Source RE-SAC, positive shift & -195.5 & 3828.1 & $[2969.4,4686.8]$ \\
Source RE-SAC, no shift & 2010.1 & 1622.5 & $[615.9,2629.2]$ \\
\bottomrule
\end{tabular}
\end{table}

The full source arm degraded on all five seeds despite reaching the complete interaction budget and applying more than 99.95\% of its critic calls. The final seed109003 log showed Q prediction mean 9735.2 and head standard deviation 134.1, versus 141.0 and 0.9 for its no-shift pair. Thus the poor return is not specific to the earlier JAX port or explained by extensive skipped updates; this log alone does not identify a causal failure mechanism. The original trainer and networks were retained, but the block collector, Brax physics and persistent-coordinate protocol differ from the original Gym-paper experiment. This comparison addresses the source algorithm on the common benchmark, not reproduction of the original paper's performance.

\section{Independent information-matched Reacher protocol}
\label{app:reacher_protocol}

\subsection{Task, training and causal information}

The independent confirmation used native Brax Reacher with the generalized backend, its unmodified observation, reward and termination maps, and an explicit 200-step time limit. Only the action coordinates changed: the four gains were $(+1,+1)$, $(-1,+1)$, $(+1,-1)$ and $(-1,-1)$, with reference mode 0. Modes persisted for 50 steps, the simulator state was preserved at switches, and independent Gaussian actuator noise had standard deviation 0.02 before clipping. Stationary evaluation kept one mode for the entire episode.

Four controllers were trained independently per policy seed. Reference SAC trained in mode 0 from scratch; robust SAC and the two conditioned SAC arms trained across persistent modes. Each ran 400 iterations of 4,000 interactions and 250 updates, totaling 1.6M interactions and 100,000 updates. The learner used two critics, hidden width 256, batch size 256, learning rate $3\times10^{-4}$, and the final checkpoint, with no validation checkpoint selection. Physics sampling and online inverse inference ran on CPU; updates ran on GPU. All training and evaluation used JAX 0.4.38, Brax 0.10.5, Flax 0.10.4 and MuJoCo 3.1.6.

The conditioned actor received the 11-dimensional native observation and four-dimensional pre-action belief. Its 70,916 parameters exceeded the 69,892 parameters of the reference and robust actors by 1,024 input weights. Conditioned SAC directly selected the command, without analytic compensation. In both training and evaluation, transition $(s_t,a_t,r_t,s_{t+1})$ updated only $b_{t+1}$; the action at $t$ used $b_t$. Belief was reset to uniform at episode boundaries. The filtered arms used $(h,c,\delta)=(0.002,1,0.98)$; the instantaneous arms normalized only the latest completed-transition evidence, without the sticky temporal prior. Compensation and conditioned SAC shared the frozen inverse model and belief-update definition within each matched pair, not identical realized histories.

\subsection{Estimator and accounting}

The Reacher inverse model was fitted separately from the HalfCheetah estimator. A disjoint reference-policy provider, seed 94003, received 1.6M interactions and 100,000 updates. Exploration around its compensated commands collected 178,000 labeled training transitions and 82,000 validation transitions, totaling 260,000 calibration interactions. Training and validation event seeds were 205001 and 205101. Five width-256 inverse heads were trained from scratch for 4,500 updates with batch size 256, AdamW learning rate $10^{-4}$, weight decay $10^{-5}$ and global gradient-norm clipping at 10. The retained estimator seed was 95103, with 704,010 parameters and residual variances $(0.01771,0.03775)$. It was frozen before the independent confirmations and was not refitted on their policies or events.

These historical provider, calibration and fitting costs are charged equally to compensation and belief-conditioned SAC. Memoryless robust SAC has no estimator cost. The comparison therefore matches controller interactions and estimator information between the conditioned routes, not total data against the memoryless control or actor parameter counts across interfaces.

\subsection{Independent seeds, endpoints and controls}

The final matched confirmation used policy seeds
$\{99007,99027,99047,99067,99087\}$, switching-event seeds
$\{209601,209617,209633\}$, and stationary-event base 209701.
There were five paired episodes per switching stream and five per stationary mode. A seed's mean switching return averaged its 15 episodes. Two-sided paired $t$ intervals used five policy-seed differences and four degrees of freedom, not episode-level replication. The primary criterion was a positive lower 95\% confidence bound for filtered compensation minus filtered conditioned SAC. Win counts were reported without an all-five-win requirement. The instantaneous matched pair was secondary and could not replace a failed primary. The intervals are marginal, without multiple-comparison adjustment.

The full audit retained true-mode/native-reference, uncompensated, all four fixed-compensation, memoryless robust, filtered compensation, instantaneous compensation, fixed inverse head 0, and both conditioned routes. Across five seeds, 13 routes and 15 episodes, all 975 switching episodes retained 200 steps and three observed switches with no termination. Seven routes were also evaluated in all four stationary modes, totaling 700 episodes. True-mode and native-reference returns matched exactly. The best fixed switching route at the group means, mode 3, obtained $-62.913$, compared with $-6.759$ for filtered compensation.

Earlier cohorts remain separate: three-seed mechanism development, an independent five-seed compensation-versus-memoryless comparison, and a three-seed same-information development comparison. The latter informed the frozen design but was not pooled with the final five seeds. The earlier five-seed cohort found that filtering improved mode accuracy while reducing return; the final matched cohort did not establish a filtering or ensemble advantage either (Section~\ref{sec:reacher_confirmation}).

For comparison, a separate five-seed HalfCheetah same-information diagnostic obtained a compensation-minus-conditioned-SAC mean difference of $+591.3$, with 95\% CI $[-541.7,1724.4]$ and three of five wins. Its recorded execution profiles do not supply a homogeneous historical runtime snapshot. This inconclusive diagnostic does not extend the positive Reacher matched-comparison claim to HalfCheetah or replace the original ten-seed three-comparator result.

\section{Estimator robustness and coordinate extension}
\label{app:coordinate_extension}

\subsection{Crossed Reacher estimator initializations}

The follow-up crossed the existing policy seeds $\{99007,99027,99047,99067,99087\}$ with inverse-ensemble seeds $\{95103,95121,95139\}$. These models used independent initializations, the same provider 94003, 178,000 prescribed training transitions, 82,000 validation transitions and 4,500 updates each. The primary-model conditioned policies already existed; ten additional conditioned-SAC controllers were trained from scratch for the other two estimators, with 1.6M interactions and 100,000 updates each. Each matched comparator used its own estimator throughout causal training and evaluation. Model indices 1--3 in Table~\ref{tab:coordinate_extension} denote the listed estimator seeds in order.

Switching-event seeds were $\{211601,211617,211633\}$, with five episodes per stream; stationary-event base 211701 supplied five episodes per fixed mode. Statistical inference used the five paired policy-seed means separately for each model, with four degrees of freedom. The model-by-policy cells and episodes were not treated as independent observations. The frozen filter, noise, training budget and 200-step horizon were unchanged. The oracle/native reference averaged $-4.824$ and robust SAC $-26.274$. These are robustness results on the original policy cohort under new event streams, not another independent five-policy confirmation.

\subsection{Independent signed-permutation study}

The four transforms were fixed before training:
\begin{equation}
 G_0=I,\qquad
 G_1=\begin{pmatrix}0&1\\-1&0\end{pmatrix},\qquad
 G_2=-I,\qquad
 G_3=\begin{pmatrix}0&-1\\1&0\end{pmatrix}.
\end{equation}
Their inverse is the transpose, not always the matrix itself. Reference mode was 0. Robot physics, observation, reward, noise standard deviation 0.02, 50-step mode persistence and 200-step episodes remained unchanged. Simulator state was retained at switches. All compared controllers and audits used JAX 0.4.38, Brax 0.10.5, Flax 0.10.4 and MuJoCo 3.1.6.

Policy seeds were 117003, 117023, 117043, 117063 and 117083; estimator seeds were 118003, 118023 and 118043. Independent provider seed 116003 received the same training budget as the compared controllers. Reference, robust and conditioned SAC were each trained independently for every policy seed at 1.6M interactions and 100,000 updates. Each estimator received 178,000 training/82,000 validation transitions and 4,500 updates. Identical prescribed calibration data were regenerated on each learner node; repeated collection counts toward compute, not independent information. Calibration-event seeds 210001 and 210101 were disjoint from evaluation. Estimator 118003 was preselected for the primary conditioned-SAC comparison; the other models were secondary causal audits, not additional matched conditioned arms.

The new event streams and episode counts matched the initialization audit above. Every switching route retained 200 steps and three observed switches, and every stationary route retained 200 steps without switches; no episode terminated early. True-mode and native-reference returns were equal episode by episode. The fixed-context switching means for modes 0--3 were $-48.110$, $-46.194$, $-47.699$ and $-46.122$, all below oracle $-4.986$ on every policy seed. The primary BAPR-minus-robust difference was $19.393$ ($[17.653,21.132]$), with five of five wins. Both secondary models also beat robust control on all seeds, while all three causal models remained below oracle. The primary model's switching accuracy was 78.85\%; its median first-correct delay of 12 steps describes only the 223 detected switches, with two of 225 unresolved. The corresponding conditioned route detected 203 switches, with median delay 14 and 22 unresolved.

\subsection{Retrospective HalfCheetah calibration sensitivity}
\label{app:hc_calibration_sensitivity}

Three existing independently initialized estimators, seeds $\{94103,94121,94139\}$, were evaluated without refitting on nine historical policies $\{88021,88039,88057,88079,88103,88121,88139,88157,88179\}$. Provider 88003 was excluded. The models used the same architecture and 4,500 updates on 178,000 training/82,000 validation transitions, but a different calibration recipe from the original two-stage estimator. Historical model-training dependency versions were not recorded; this audit's evaluation runtime was pinned. It is a calibration sensitivity study, not nine new confirmation seeds or replication of the original estimator recipe.

\begin{table}[ht]
\centering
\small
\caption{Historical HalfCheetah calibration sensitivity across nine frozen policies. Delays are medians among detected switches; unresolved counts retain the censored cases.}
\label{tab:hc_calibration_sensitivity}
\begin{tabular}{lrrrr}
\toprule
Estimator & Switching & Stationary & Delay & Unresolved \\
\midrule
94103 & 2702.8 & 3255.1 & 7 & 31/405 \\
94121 & 2718.5 & 3256.6 & 6 & 32/405 \\
94139 & 2773.3 & 3255.9 & 7 & 23/405 \\
Oracle/native & 3298.1 & 3313.3 & --- & --- \\
\bottomrule
\end{tabular}
\end{table}

All three models exceeded uncompensated return 369.2 on all nine policies and fell below oracle on all nine. Their paired oracle deficits had strictly negative marginal intervals. The near-oracle stationary values coexist with a larger switching loss, identifying the transient as a remaining control-performance gap without isolating which calibration component caused it. No models or cohorts were selected, pooled or substituted for the original result.

\section{Cross-environment evidence boundary}
\label{app:cross_environment}

The primary ten-seed three-comparator confirmation is specific to HalfCheetah actuator-polarity regimes. The independent five-seed Reacher confirmation supports the separate same-information comparison described in Appendix~\ref{app:reacher_protocol}. Other environment screens are retained below.
\begin{itemize}[leftmargin=*]
    \item \textbf{Ant.} A three-seed true-mode action-compensation audit obtained 4510.7 switching return versus 2588.6 for robust SAC, but only one seed passed the absolute zero-termination criterion. This demonstrates privileged headroom, not a safe causal deployment result.
    \item \textbf{Hopper and Walker2d.} In the tested structured-channel screen, every robust and oracle switching episode terminated before the first scheduled switch. The benchmark therefore lacked a valid adaptation window and cannot support a BAPR comparison.
    \item \textbf{Swimmer.} An independent three-seed screen used all four two-actuator sign transforms, common noise standard deviation 0.02, 250-step dwell intervals, and 8.4M interactions per robust or reference actor. All audited episodes reached 1,000 steps, and true-mode compensation exactly reproduced native reference-mode returns. Nevertheless, its mean return was 50.30 versus 55.04 for robust SAC; only one of three seeds exceeded robust control by both 10\% and five points. The registered oracle-headroom screen failed, so no causal estimator was fitted for this environment.
    \item \textbf{Bus control.} The bus domain contains persistent changes in speed and demand distributions with per-step stochastic passenger arrivals and traffic variation. RE-SAC-derived robustness is effective there, but its modes are not known invertible actuator transforms. Bus evidence is relevant motivation for stochastic non-stationarity, not direct evidence for the action-compensation mechanism.
\end{itemize}